\documentclass{article}

\usepackage{titling}

\pretitle{\begin{flushleft}\LARGE\bfseries}
\posttitle{\par\end{flushleft}}
\preauthor{\begin{flushleft}\large}
\postauthor{\par\end{flushleft}}
\predate{\begin{flushleft}}
\postdate{\par\end{flushleft}\vskip 0.5em}

\usepackage[utf8]{inputenc} 
\usepackage[T1]{fontenc}    
\usepackage{hyperref}       
\usepackage{url}            
\usepackage{booktabs}       
\usepackage{amsfonts}       
\usepackage{nicefrac}       
\usepackage{microtype,comment}      
\usepackage[utf8]{inputenc} 
\usepackage[T1]{fontenc}    
\usepackage{hyperref}       
\usepackage{url}            
\usepackage{booktabs}       
\usepackage{amsfonts}       
\usepackage{nicefrac}       
\usepackage{microtype}      
\usepackage{amssymb}
\usepackage{amsmath}
\usepackage{amsthm}
\usepackage{graphicx}
\usepackage[linesnumbered,ruled]{algorithm2e}
\usepackage[svgnames]{xcolor}
\usepackage{xfrac}
\usepackage{color}
\usepackage{wrapfig}
\usepackage{comment}
\usepackage[customcolors,shade]{hf-tikz}
\usepackage{empheq,multicol,titlesec}
\usepackage[most]{tcolorbox}
\newtcbox{\mymath}[1][]{%
    nobeforeafter, math upper, tcbox raise base,
    enhanced, colframe=blue!30!black,
    colback=gray!30, boxrule=1pt,
    #1}

\usepackage{multirow}
\usepackage{colortbl, array}
\usepackage{hhline}

\newcommand{\highest}[1]{\textcolor{Maroon}{\mathbf{#1}}}

\definecolor{rulecolor}{RGB}{0,71,171}
\definecolor{tableheadcolor}{RGB}{204,229,255}

\newcommand{\topline}{ %
        \arrayrulecolor{rulecolor}\specialrule{0.1em}{\abovetopsep}{0pt}%
        \arrayrulecolor{tableheadcolor}\specialrule{\belowrulesep}{0pt}{0pt}%
        \arrayrulecolor{rulecolor}}
\newcommand{\midtopline}{ %
        \arrayrulecolor{tableheadcolor}\specialrule{\aboverulesep}{0pt}{0pt}%
        \arrayrulecolor{rulecolor}\specialrule{\lightrulewidth}{0pt}{0pt}%
        \arrayrulecolor{white}\specialrule{\belowrulesep}{0pt}{0pt}%
        \arrayrulecolor{rulecolor}}
\newcommand{\bottomline}{ %
        \arrayrulecolor{white}\specialrule{\aboverulesep}{0pt}{0pt}%
        \arrayrulecolor{rulecolor} %
        \specialrule{\heavyrulewidth}{0pt}{\belowbottomsep}}%

\newcommand{\argmin}{\operatornamewithlimits{argmin}}

\newcommand\numberthis{\addtocounter{equation}{1}\tag{\theequation}}
\newtheorem{theorem}{Theorem}

\newtheorem{definition}{Definition}
\newtheorem{proposition}[theorem]{Proposition}

\makeatletter
\newsavebox{\mybox}\newsavebox{\mysim}
\newcommand{\distras}[1]{%
  \savebox{\mybox}{\hbox{\kern3pt$\scriptstyle#1$\kern3pt}}%
  \savebox{\mysim}{\hbox{$\sim$}}%
  \mathbin{\overset{#1}{\kern\z@\resizebox{\wd\mybox}{\ht\mysim}{$\sim$}}}%
}

\title{ManifoldNet: A Deep Network Framework for Manifold-valued Data}

\author{
 Rudrasis~Chakraborty\textsuperscript{$\dagger$}, Jose~Bouza\textsuperscript{$\dagger$}, Jonathan~Manton \textsuperscript{$+$} and Baba~C.~Vemuri\textsuperscript{$\dagger$} \thanks{This research was funded in part by the NSF grant IIS-1525431 and IIS-1724174 to BCV.} \\
\textsuperscript{$\dagger$} Department of CISE, University of Florida, FL 32611, USA\\
{\tt\small \{rudrasischa, baba.vemuri\}@gmail.com, josebouza@ufl.edu, jonathan.manton@ieee.org}
}
\date{}
\begin{document}

\maketitle

\begin{abstract}
Deep neural networks have become the main work horse for many tasks
involving learning from data in a variety of applications in Science
and Engineering. Traditionally, the input to these networks lie in a
vector space and the operations employed within the network are well
defined on vector-spaces. In the recent past, due to technological
advances in sensing, it has become possible to acquire manifold-valued
data sets either directly or indirectly. Examples include but are not
limited to data from omnidirectional cameras on automobiles, drones
etc., synthetic aperture radar imaging, diffusion magnetic resonance
imaging, elastography and conductance imaging in the Medical Imaging
domain and others. Thus, there is need to generalize the deep neural
networks to cope with input data that reside on curved manifolds where
vector space operations are not naturally admissible. In this paper,
we present a novel theoretical framework to cope with manifold-valued
data inputs.  In doing this generalization, we draw parallels to the
widely popular convolutional neural networks (CNNs).  We call our
networks, ManifoldNet.

As in vector spaces where convolutions are equivalent to computing the
weighted mean of functions, an analogous definition for manifold-valued
data can be constructed involving the computation of the weighted
Fr\'{e}chet Mean (wFM). To this end, we present a provably convergent
recursive computation of the weighted Fr\'{e}chet Mean (wFM) of the
given data, where the weights makeup the convolution mask, to be
learned. Further, we prove that the proposed wFM layer achieves a
contraction mapping and hence the ManifoldNet does not need the
additional non-linear ReLU unit used in standard CNNs. Operations such
as pooling in traditional CNN are no longer necessary in this setting
since wFM is already a pooling type operation. \emph{Analogous to the
  equivariance of convolution operation in Euclidean space to
  translations, we prove that the wFM is equivariant to the action of
  the group of isometries admitted by the Riemannian manifold on which
  the data reside}. This equivariance property facilitates weight
sharing within the network.  We present experiments, using the
ManifoldNet framework, to achieve video classification and image
reconstruction using an auto-encoder+decoder setting. Experimental
results demonstrate the efficacy of ManifoldNets in the context of
classification and reconstruction accuracy.

\end{abstract}

\section{Introduction}\label{intro}

Convolutional neural networks (CNNs) have attracted enormous
attention in the past decade due to their significant success in
Computer Vision, Speech Analysis and other fields. CNNs were pioneered
by Lecun in \cite{LeCun1998} and gained much popularity ever since
their significant success on Imagenet data reported in Krizhevsky et
al. \cite{Krizhevsky2012}. Most of the deep networks and in particular
CNNs have traditionally been restricted to dealing with data residing
in vector spaces. However, in the past few years, there is growing
interest in generalizing the CNNs and deep networks in general to data
that reside in non-Euclidean spaces and possibly on smooth
manifolds. In this context, at the outset, it would be useful to
categorize problems into those that involve data as samples of
real-valued functions defined on a manifold and those that are simply
manifold-valued and hence are sample points on a manifold.

In the context of input data being samples of functions on smooth
manifolds, recently there has been a flurry of activity in developing
methods that can cope specifically with samples of functions on a
sphere i.e., spherical functions that are encountered in many
applications such as, omnidirectional cameras on drones, robots etc.,
meteorological data and many others.  The key property that allows
learned weight sharing in CNNs is the equivariance to
translations. The simplest technique to achieve equi-variance is via
data augmentation \cite{Krizhevsky2012,Dieleman2015}. Cascade of
wavelet transforms to achieve equivariance was shown in
\cite{Bruna2013,Oyallon2015}. In \cite{Gens2014}, authors describe
`Symnet', which achieves invariance to symmetry group actions.
Equivariance to discrete group actions was achieved through parameter
sharing in \cite{ravanbakhsh2017equivariance}. For the case of data on
a spherical domain, one considers exploiting equivariance to the
natural symmetry group operation on the sphere, which is the rotation
group. Several research groups recently reported spherical-CNNs
(SCNNs) to accommodate such an equivariance in defining the
convolution of functions
\cite{worrall2017harmonic,cohen2016group,cohen2018spherical}. In
another recent work \cite{esteves2017polar}, authors describe a polar
transformer network, which is equivariant to rotations and scaling
transformations. By combining this with what is known in literature as
a spatial transformer \cite{jaderberg2015spatial}, they achieve
equivariance to translations as well. More generally, equivariance of
convolution operations to symmetry group actions admitted by the
domain manifold is what is needed to this end and most recent work
reported in \cite{chakraborty2018h,kondor2018generalization} achieves
this for homogeneous Riemannian manifolds, which are more general than
Lie groups.

In this paper we will consider the second problem, namely, when the
input data are sample points on known Riemannian manifolds for
example, the manifold of symmetric positive definite matrices, $P_n$,
the special orthogonal group, $\textsf{SO}(n)$, the n-sphere,
$\mathbf{S}^n$, the Grassmannian, $\textsf{Gr}(p,n)$, and
others. There is very little prior work that we are aware of on deep
networks that can cope with input data samples residing on these
manifolds with the exception of
\cite{huang2016building,huang2017riemannian}. In
\cite{huang2016building}, authors presented a deep network
architecture for classification of hand-crafted features residing on a
Grassmann manifold that form the input to the network. Their network
architecture contains 3 blocks, the first block involves a full rank
mapping (to be learned) of the input orthonormal matrices followed by
re-orthonormalization operation achieved via the well known
$QR$-decomposition followed by the well known projection operation
involving a projection metric on the Grassmannian. The second block
consists of a pooling block which involves computing the arithmetic
mean that is valid since the projection matrices lie in a Euclidean
space. The third block is a standard fully connected layer with
softmax operations whose output is amenable to classification.  In
\cite{huang2017riemannian}, authors present a deep network
architecture for data residing on $P_n$. They proposed three types of
layers in the architecture. The first involves a bilinear mapping of
the input SPD matrices on $P_n$ to the SPD matrices on $P_{n+1}$. The
second type of layer is supposed to mimic the ReLU type operation
which is achieved by using the `max' between the scaled identity and
the true diagonal eigen-value matrix in order to reconstruct a new SPD
matrix from the input SPD matrix to this nonlinear ReLU type
operation. Finally, they use the Log-Euclidean framework described in
\cite{Arsigny2006} which maps points on $P_n$ into the tangent space
of $P_n$ by using the Riemannian $Log$ map \cite{DoCarmo1992} and then
uses the standard Euclidean computations since the tangent space is
isomorphic to the Euclidean space. This layer is then followed by a
fully connected layer and softmax operations as in the conventional
CNNs. Note that in both these works, the architecture does not
introduce any convolution or equivalent operations on the Grassmannian
or $P_n$.  Further, it does not use the natural invariant metric or
intrinsic operations on the Grassmannian or the $P_n$ in the
blocks. Using intrinsic operations within the layers guarantees that
the result remains on the manifold and hence does not require any
projection (extrinsic) operations to ensure the result lies in the
same space. Such extrinsic operations are known to yield results that
are susceptible to significant inaccuracies when the data variance is
large.  For example, the Fr\'{e}chet mean when computed using the
Log-Euclidean metric -- which is an extrinsic metric -- proves to be
inaccurate compared to the intrinsic mean when data variance is large
as was demonstrated in \cite{Salehian2015}. Further, since there is no
operation analogous to a convolution defined for data residing on
these manifolds in their network, it is not a generalization to the
CNN and as a consequence does not consider equivariance property to
the action of the group of isometries admitted by the manifold. 

There are several deep networks reported in literature to deal with
cases when data reside on 2-manifolds encountered in Computer Vision
and Graphics for modeling shapes of objects. Some of these are based
on graph-based representations of points on the surfaces in 3D and a
generalization of CNNs to graphs
\cite{henaff2015deep,defferrard2016convolutional}. There is also
recent work in \cite{masci2015geodesic} where the authors presented a
deep network called geodesic CNN (GCNN), where convolutions are
performed in local geodesic polar charts constructed on the
manifold. For more literature on deep networks for data on
2-manifolds, we refer the interested reader to a recent survey paper
\cite{bronstein2017geometric} and references therein.

In this paper, we present a novel deep learning framework called
ManifoldNets.  These are the analogues of CNNs and can cope with input
data residing on Riemannian manifolds. The intuition in defining the
analog relies on the equivariance property. Note that convolutions in
vector spaces are equivariant to translations and further, it is easy
to show that they are equivalent to computing the weighted mean
\cite{goh2011nonparametric}. 
Thus, for the case of data that are manifold-valued, we can define the
analogous operation of a weighted Fr\'{e}chet mean (wFM) and prove
that it is equivariant to the action of the natural group of
isometries admitted by the manifold. This is precisely what is
achieved in a subsequent section. Thus, our key contributions in this
work are: (i) we define the analog of convolution operations for
manifold-valued data to be one of estimating the wFM for which we
present a provably convergent, efficient and recursive estimator. (ii)
A proof of equivariance of wFM to the action of group of isometries
admitted by the Riemannian manifold on which the data reside. This
equivariance allows the network to share weights within the
layers. (iii) A novel deep architecture involving the Riemannian
counterparts to the conventional CNN units. (iv) Several real data
experiments on classification and reconstruction demonstrating the
performance of the ManifoldNet.

Rest of the paper is organized as follows. In section \ref{theory} we
present the main theoretical result showing the equivariance of wFM to
group actions admitted by the Riemannian manifold.  Section
\ref{theory} also contains a detailed description of the proposed
network architecture. In section \ref{results} we present the
experimental results and conclude in section \ref{conc}.

\section{Group action equivariant network for manifold-valued data}\label{theory}

In this section, we will define the equivalent of a convolution
operation on Riemannian manifolds. Before formally defining such
an operation and building the deep neural network for
the manifold-valued data, dubbed as ManifoldNet, we first present some
relevant concepts from differential geometry that will be used in the
rest of the paper.
\vspace*{0.1cm}

\textbf{Preliminaries.} Let $(\mathcal{M},g^\mathcal{M})$ be a
orientable complete Riemannian manifold with a Riemannian metric
$g^\mathcal{M}$, i.e., $(\forall x \in
\mathcal{M})\:g^\mathcal{M}_x:T_x{\mathcal{M}}\times T_x{\mathcal{M}}
\rightarrow \mathbf{R}$ is a bi-linear symmetric positive definite
map, where $T_x\mathcal{M}$ is the tangent space of $\mathcal{M}$ at
$x\in \mathcal{M}$. Let $d: \mathcal{M} \times \mathcal{M} \rightarrow
[0,\infty)$ be the metric (distance) induced by the Riemannian metric
  $g^\mathcal{M}$.  With a slight abuse of notation we will denote a
  Riemannian manifold $(\mathcal{M},g^\mathcal{M})$ by $\mathcal{M}$
  unless specified otherwise. Let $\Delta$ be the supremum of the
  sectional curvatures of $\mathcal{M}$.
\begin{definition}
\label{theory:def1}
Let $p \in \mathcal{M}$ and $r > 0$. Define $\mathcal{B}_r(p) =
\left\{ q \in \mathcal{M} | d(p,q) < r \right\}$ to be a open ball at
$p$ of radius $r$.
\end{definition}
\begin{definition} \cite{Groisser2004}
\label{theory:def2}
Let $p \in \mathcal{M}$. The {\it local injectivity radius} at $p$,
$r_{\text{inj}}(p)$, is defined as $r_{\text{inj}}(p) = \sup \left\{ r
| \text{Exp}_p : (\mathcal{B}_r(\mathbf{0}) \subset T_p \mathcal{M} )
\rightarrow \mathcal{M} \text{ is defined and is a
  diffeomorphism}\right.$ $\left. \text{onto its image} \right\}$. The
{\it injectivity radius} \cite{manton2004globally} of $\mathcal{M}$ is defined as
$r_{\text{inj}}(\mathcal{M}) = \inf_{p \in \mathcal{M}}
\left\{ r_{\text{inj}}(p)\right\}$.
\end{definition}
Within $\mathcal{B}_r(p)$, where $r \leq r_{\text{inj}}(\mathcal{M})$,
the mapping $ \text{Exp}^{-1}_p: \mathcal{B}_r(p) \rightarrow
\mathcal{U} \subset T_p \mathcal{M} $, is called the inverse
Exponential/ Log map.
\begin{definition} \cite{Kendall1990}
\label{theory:def3}
An open ball $\mathcal{B}_r(p)$ is a {\it regular geodesic ball} if $r
< r_{\text{inj}}(p)$ and $r < \pi/\left(2 \Delta^{1/2} \right)$.
\end{definition}
In Definition \ref{theory:def3} and below, we interpret
$1/\Delta^{1/2}$ as $\infty$ if $\Delta \leq 0$. It is well known
that, if $p$ and $q$ are two points in a regular geodesic ball
$\mathcal{B}_r(p)$, then they are joined by a unique geodesic within
$\mathcal{B}_r(p)$ \cite{Kendall1990}.
\begin{definition}\cite{Chavel2006}
\label{theory:def4}
$\mathcal{U} \subset \mathcal{M}$ is strongly convex if for all $p, q
\in \mathcal{U}$, there exists a unique length minimizing geodesic
segment between $p$ and $q$ and the geodesic segment lies entirely in
$\mathcal{U}$.
\end{definition}

\begin{definition} \cite{Groisser2004}
\label{theory:def5}
Let $p \in \mathcal{M}$. The {\it local convexity radius} at $p$,
$r_{\text{cvx}}(p)$, is defined as $r_{\text{cvx}}(p) = \sup\left\{ r
\leq r_{\text{inj}}(p) | \mathcal{B}_r(p) \text{ is strongly
  convex}\right\}$. The {\it convexity radius} of $\mathcal{M}$ is
defined as $r_{\text{cvx}}(\mathcal{M}) = \inf_{p \in \mathcal{M}}
\left\{ r_{\text{cvx}}(p)\right\}$.
\end{definition}

For the rest of the paper, we will assume that the samples on
$\mathcal{M}$ lie inside an open ball $U = \mathcal{B}_r(p)$ where
$r=\min\left\{r_{\text{cvx}}(\mathcal{M}),
r_{\text{inj}}(\mathcal{M})\right\}$, for some $p \in \mathcal{M}$,
unless mentioned otherwise. Now, we are ready to define the operations
necessary to develop the ManifoldNet.

\subsection{$\textsf{wFM}$ on $\mathcal{M}$ as a generalization of convolution}
\label{sec21}
Let $\left\{X_i\right\}_{i=1}^N$ be the manifold-valued samples on
$\mathcal{M}$.  We define the convolution type operation on
$\mathcal{M}$ as the weighted Fr\'{e}chet mean (wFM)
\cite{Frechet1948a} of the samples $\left\{X_i\right\}_{i=1}^N$. Also,
by the aforementioned condition on the samples, the existence and
uniqueness of FM is guaranteed \cite{Afsari2011}.  As mentioned
earlier, it is easy to show (see \cite{goh2011nonparametric}).  that
convolution $\psi^{*}=b \star a$ of two functions $a:X \subset \mathbf{R}^n
\rightarrow \mathbf{R}$ and $b:X \subset \mathbf{R}^n \rightarrow \mathbf{R}$ can be formulated as
computation of the weighted mean $\psi^{*} = argmin_\psi \int
a(\textbf{u})(\psi-\widetilde{b}_{\textbf{u}})^2d\textbf{u}$, where, $\forall \textbf{x} \in X, \widetilde{b}_{\textbf{u}}(\textbf{x}) = b\left(\textbf{u}+\textbf{x} \right)$ and $\int a(\textbf{x})d\textbf{x} =1$. Here, $f^2$ for any function $f$ is defined pointwise. Further, the defining
property of convolutions in vector spaces is the linear translation
equivariance. Since weighted mean in vector spaces can be generalized
to wFM on manifolds and further, wFM can be shown (see below) to be
equivariant to symmetry group actions admitted by the manifold, we
claim that wFM is a generalization of convolution operations to
manifold-valued data.  

Let $\left\{w_i\right\}_{i=1}^N$ be the weights such that they satisfy
the convexity constraint, i.e., $\forall i, w_i > 0$ and $\sum_i w_i =
1$, then wFM, $\textsf{wFM}\left(\left\{X_i\right\},
\left\{w_i\right\}\right)$ is defined as:
\begin{align}
\label{theory:eq1}
\textsf{wFM}\left(\left\{X_i\right\}, \left\{w_i\right\}\right) = \argmin_{M \in \mathcal{M}} \sum_{i=1}^N w_i d^2\left(X_i, M\right)
\end{align}

Analogous to the equivariance property of convolution translations in
vector spaces, we will now proceed to show that the $\textsf{wFM}$ is
equivariant under the action of the group of isometries of
$\mathcal{M}$. We will first formally define the group of isometries
of $\mathcal{M}$ (let us denote it by $G$) and then define the
equivariance property and show that $\textsf{wFM}$ is $G$-equivariant.

\begin{definition}[{\bf Group of isometries of $\mathcal{M}$}]
\label{theory:def6}
{A diffeomorphism $\phi: \mathcal{M} \rightarrow \mathcal{M}$ is an
  isometry if it preserves distance, i.e., $d\left(\phi\left(x\right),
  \phi\left(y\right)\right) = d\left(x, y\right)$. The set
  $I(\mathcal{M})$ of all isometries of $\mathcal{M}$ forms a group
  with respect to function composition. Rather than write an isometry
  as a function $\phi$, we will write it as a group action.
  Henceforth, let $G$ denote the group $I(\mathcal{M})$, and for $g
  \in G$, and $x \in \mathcal{M}$, let $g.x$ denote the result of
  applying the isometry $g$ to point $x$.}
\end{definition}

Clearly $\mathcal{M}$ is a $G$ set (see \cite{dummit2004abstract} for
definition of a $G$ set). We will now define equivariance and show
that $\textsf{wFM}$, is $G$-equivariant.

\begin{definition} [{\bf Equivariance}]
\label{theory:def7}
Let $X$ and $Y$ be $G$ sets. Then, $F: X \rightarrow Y$ is said to be
{\it $G$-equivariant} if
\begin{align}
F(g.x) = g.F(x)
\end{align}
for all $g \in G$ and all $x \in X$.
\end{definition}

Let $U \subset \mathcal{M}$ be an open ball inside which FM exists and
is unique, let $P$ be the set consists of all possible finite
subsets of $U$.

\begin{theorem}
\label{theory:thm1}
Given $\left\{w_i\right\}$ satisfying the convex constraint, let $F: P
\rightarrow U$ be a function defined by $\left\{X_i\right\} \mapsto
\textsf{wFM}\left(\left\{X_i\right\},
\left\{w_i\right\}\right)$. Then, $F$ is $G$-equivariant.
\end{theorem}
\begin{proof}
{ Let $g \in G$ and $\left\{X_i\right\}_{i=1}^N \in P$, now, let $M^* =
\textsf{wFM}\left(\left\{X_i\right\}, \left\{w_i\right\}\right)$, as
$g.F\left(\left\{X_i\right\}\right) = g.M^*$, it suffices to show
$g.M^*$ is $\textsf{wFM}\left(\left\{g.X_i\right\},
\left\{w_i\right\}\right)$ (assuming the existense and uniqueness of $\textsf{wFM}\left(\left\{g.X_i\right\},
\left\{w_i\right\}\right)$ which is stated in the following claim). 
\ \\
{\bf Claim:} Let $U=\mathcal{B}_r\left(p\right)$ for some $r >0$ and $p \in \mathcal{M}$. Then, $\left\{g.X_i\right\} \subset \mathcal{B}_r\left(g.p\right)$ and hence $\textsf{wFM}\left(\left\{g.X_i\right\},
\left\{w_i\right\}\right)$ exists and is unique.}

Let $\widetilde{M}$ be
$\textsf{wFM}\left(\left\{g.X_i\right\}, \left\{w_i\right\}\right)$. Then, 
\begin{align*}
\sum_{i=1}^N w_i d^2\left(g.X_i, \widetilde{M}\right) &= \sum_{i=1}^N w_i d^2\left(X_i, g^{-1}.\widetilde{M}\right)
\end{align*}
Since, $M^* = \textsf{wFM}\left(\left\{X_i\right\},
\left\{w_i\right\}\right)$, hence, $M^* = g^{-1}.\widetilde{M}$, i.e.,
$\widetilde{M} = g.M^*$. Thus, $g.M^* =
\textsf{wFM}\left(\left\{g.X_i\right\}, \left\{w_i\right\}\right)$,
which implies $F$ is $G$-equivariant.
\end{proof}

Now we give some examples of $\mathcal{M}$ with the corresponding
group of isometries $G$. Let $\mathcal{M} = \textsf{SPD}(n)$ (the
space of $n\times n$ symmetric positive-definite matrices). Let $d$ be
the Stein metric on $\textsf{SPD}(n)$. Then, the group of isometries
$G$ is $\mathsf{O}(n)$ (the space of $n\times n$ orthogonal
matrices). A class of Riemannian manifolds on which $G$ acts
transitively are called Riemannian homogeneous spaces. We can see that
on a Riemannian homogeneous space $\mathcal{M}$, $\textsf{wFM}$ is
$G$-equivariant. Equipped with a $G$-equivariant operator on
$\mathcal{M}$, we can claim that the $\textsf{wFM}$ (defined above) is
a valid convolution operator since group equivariance is a unique
defining property of a convolution operator. The rest of this
subsection will be devoted to developing an efficient way to compute
$\textsf{wFM}$.  Let $\omega^{\mathcal{M}}> 0$ be the Riemannian
volume form. Let $p_{\mathbf{X}}$ be the probability density of a
$U$-valued random variable $\mathbf{X}$ with respect to
$\omega^{\mathcal{M}}$ on $U \subset \mathcal{M}$, so that
$\textsf{Pr}\left( X \in \mathfrak{A} \right) = \int_{\mathfrak{A}}
p_X(Y) \omega^{\mathcal{M}}\left(Y\right)$ for any Borel-measurable
subset $\mathfrak{A}$ of U.  Let $Y \in U$, we can define the
expectation of the real valued random variable $d^2(,Y): U \rightarrow
\mathbf{R}$ by $E\left[d^2(,Y)\right] = \int_{U} d^2(X,Y)
p_{\mathbf{X}}(X) \omega^{\mathcal{M}}(X)$. Now, let $w: U \rightarrow
(0, \infty)$ be an integrable function and $\int_{U} w\left(X\right)
\omega^{\mathcal{M}}\left(X\right)=1$.
  
Then, observe that,
\begin{align*}
\label{theory:eq2}
E_{w}\left[d^2(,Y)\right] &:= \int_{U} w(X) d^2(X,Y) p_X(X)
\omega^{\mathcal{M}}(X) \\ &= C\; \int_{U} d^2(X,Y) \widetilde{p}_X(X)
\omega^{\mathcal{M}}(X) \\ &= C\;
\widetilde{E}\left[d^2(,Y)\right]. \numberthis
\end{align*}

Here, $\widetilde{p}_X$ is the probability density corresponding to
the probability measure $\widetilde{\textsf{Pr}}$ defined by,
$\widetilde{\textsf{Pr}}\left(X \in \mathfrak{X}\right) =
\int_{\mathfrak{X}} \widetilde{p}_X(Y) \omega^{\mathcal{M}}(Y)$ $:=
\int_{\mathfrak{X}} \frac{1}{C}\; p_X(Y) w(Y)\omega^{\mathcal{M}}(Y)$,
where, $\mathfrak{X}$ lies in the Borel $\sigma$-algebra over $U$ and
$C = \int_{U} p_X(Y) w(Y)\omega^{\mathcal{M}}(Y)$. Note that the
constant $C>0$, since $p_X$ is a probability density, $w>0$ and
$\mathcal{M}$ is orientable. Thus, $E_{w}\left[d^2(,Y)\right]$ with
respect to $p_X$ is proportional to
$\widetilde{E}\left[d^2(,Y)\right]$ with respect to $\widetilde{p}_X$.

Now, we will prove that the support of $p_X$ is same as
$\widetilde{p}_X$.

\begin{proposition}
\label{theory:prop0}
$\textsf{supp}\left(p_X\right) = \textsf{supp}\left(\widetilde{p}_X\right)$.
\end{proposition}
\begin{proof}
Let $X \in \textsf{supp}\left(p_X\right)$, then, $p_X\left(X\right) >
0$. Since, $w(X) > 0$, hence, $\widetilde{p}_X\left(X\right)> 0$ and
thus, $X \in \textsf{supp}\left(\widetilde{p}_X\right)$. On the other
hand, assume $\widetilde{X}$ to be a sample drawn from
$\widetilde{p}_X$. Then, either $p_X\left(\widetilde{X}\right) = 0$ or
$p_X\left(\widetilde{X}\right) > 0$. If,
$p_X\left(\widetilde{X}\right) = 0$, then,
$\widetilde{p}_X\left(\widetilde{X}\right) = 0$ which contradicts our
assumption. Hence, $p_X\left(\widetilde{X}\right) > 0$, i.e.,
$\widetilde{X} \in \textsf{supp}\left(p_X\right)$. This concludes the
proof.
\end{proof}

\begin{proposition}
\label{theory:prop1}
$\textsf{wFE}\left(\textsf{X}, w\right) =
\textsf{FE}\left(\widetilde{\textsf{X}}\right)$
\end{proposition}
\begin{proof}
Let $\textsf{X}$ and $\widetilde{\textsf{X}}$ be the $\mathcal{M}$
valued random variable following $p_X$ and $\widetilde{p}_X$
respectively. We define the weighted Fr\'{e}chet expectation (wFE) of
$\textsf{X}$ as:
\begin{equation*}
\textsf{wFE}\left(\textsf{X}, w\right) = \argmin_{Y \in \mathcal{M}}
\int_{\mathcal{M}} w(X) d^2(X,Y) p_X(X) \omega^{\mathcal{M}}(X)
\end{equation*}
Using \eqref{theory:eq2}, we get
$\textsf{FE}\left(\widetilde{\textsf{X}}\right) =
\textsf{wFE}\left(\textsf{X}, w\right)$, as $C$ is independent of the
choice of $Y$, which concludes the proof.
\end{proof}

Let $\left\{X_i\right\}_{i=1}^N$ be samples drawn from $p_X$ and
$\left\{\widetilde{X}_i\right\}_{i=1}^N$ be samples drawn from
$\widetilde{p}_X$. In order to compute $\textsf{wFM}$, we will now
present an online algorithm (inductive FM Estimator -- dubbed
\textsf{iFME}). Given, $\left\{X_i\right\}_{i=1}^N \subset U$ and
$\left\{w_i:=w\left(X_i\right)\right\}_{i=1}^N$ such that $\forall i,
w_i > 0$, the $n^{th}$ estimate, $M_{n}$ of
$\textsf{wFM}\left(\left\{X_i\right\},\left\{w_i\right\}\right)$ is
given by the following recursion:
\begin{align}
\label{theory:eq3}
M_1 = X_1 \:\:\qquad \qquad \qquad M_{n} =
\Gamma_{M_{n-1}}^{X_n}\left(\frac{w_n}{\sum_{j=1}^n w_j}\right).
\end{align}
In the above equation, $\Gamma_X^Y: \left[0,1\right] \rightarrow U$ is
the shortest geodesic curve from $X$ to $Y$. Observe that, in general
$\textsf{wFM}$ is defined with $\sum_{i=1}^N w_i=1$, but in above
definition, $\sum_{i=1}^N w_i \neq 1$. We can normalize
$\left\{w_i\right\}$ to get $\left\{\widetilde{w}_i\right\}$ by
$\widetilde{w}_i = w_i/\left(\sum_i w_i\right)$, but then
Eq. \ref{theory:eq3} will not change as
$\widetilde{w}_n/\left(\sum_{j=1}^n \widetilde{w}_j\right) =
w_n/\left(\sum_{j=1}^n w_j\right)$. This gives us an efficient
inductive/recursive way to define convolution operation on
$\mathcal{M}$. Now, we state and prove that the proposed
$\textsf{wFM}$ estimator is consistent.  
\begin{proposition}
Using the above notations and assumptions, let
$\left\{X_i\right\}_{i=1}^N$ be i.i.d. samples drawn from $p_X$ on
$\mathcal{M}$. Let the $\textsf{wFE}$ be finite. Then, $M_N$ converges
a.s. to $\textsf{wFE}$ as $N \rightarrow \infty$.
\end{proposition}
\begin{proof}
Using Proposition \ref{theory:prop1}, we know that $\exists \;
\widetilde{p}_X$ such that, $\textsf{wFE}\left(\textsf{X}, w\right) =
\textsf{FE}\left(\widetilde{\textsf{X}}\right)$. Thus, it is enough to
show the consistency of our proposed estimator when weights are
uniform. In order to prove the consistency, we will split the proof
into two cases namely, manifolds with (i) non-positive sectional
curvature and (ii) non-negative sectional curvature. The reason for
doing this split is so that we can use existing theorems in literature
for proving the result.  We will use the theorems proved in
\cite{Sturm2003a} and \cite{Bonnabel2013} for manifolds with
non-positive and non-negative sectional curvatures respectively. Note
that the proof holds only for manifolds with a uniform sign of
sectional curvatures.

\begin{theorem}[$\mathcal{M}$ has non-negative sectional curvature]
Using the above notations, if $\exists A >0$ such that, $d\left(M_n,
X_{n+1}\right)\leq A$ for all $n$. Then, $M_N$ converges a.s. to
$\textsf{wFE}$ as $N \rightarrow \infty$ (see \cite{Bonnabel2013} for
the proof).
\end{theorem}

\begin{theorem}[$\mathcal{M}$ has non-positive sectional curvature]
Using the above notations $M_N$ converges a.s. to $\textsf{wFE}$ as $N
\rightarrow \infty$ (see \cite{Sturm2003a} for the proof).
\end{theorem}
\end{proof}

\subsection{Nonlinear operation between $\textsf{wFM}$-layers for $\mathcal{M}$-valued Data}
\label{sec22}
In the traditional CNN model, we need a nonlinear function between two
convolutional layers similar to ReLU and softmax. As argued in
\cite{Mallat2016}, any nonlinear function used in CNN is basically a
contraction mapping. Formally, let $F$ be a nonlinear mapping from $U$
to $V$. Let assume, $U$ and $V$ are metric spaces equipped with metric
$d_U$ and $d_V$ respectively. Then, $F$ is a contraction mapping {\it
  iff} $\exists c < 1$ such that,
\begin{align}
d_V\left(F(x), F(y)\right) \leq c\; d_U(x, y)
\end{align}
$F$ is a non-expansive mapping \cite{Mallat2016} {\it iff}
\begin{align}
d_V\left(F(x), F(y)\right) \leq d_U(x, y)
\end{align}

One can easily see that the popular choices for nonlinear operations
like ReLU, sigmoid are in fact non-expansive mappings. Now, we will
show that the function $\textsf{wFM}$ as defined in \ref{theory:eq1},
is a contraction mapping for non-trivial choices of weights. Let
$\left\{X_i\right\}_{i=1}^N$ and $\left\{Y_j\right\}_{j=1}^M$ be the
two set of samples on $\mathcal{M}$. Without any loss of generality,
assume $N\leq M$. We consider the set $\mathcal{U}^M =
\underbrace{U \times \cdots \times U}_{M\text{
    times}}$. Clearly $\left\{Y_j\right\}_{j=1}^M \in \mathcal{U}^M$
and we embed $\left\{X_i\right\}_{i=1}^N$ in $\mathcal{U}^M$ as
follows: we construct $\left\{\widetilde{X}_i\right\}_{i=1}^M$ from
$\left\{X_i\right\}_{i=1}^N$ by defining $\widetilde{X}_i =
X_{(i-1)\text{mod}N + 1}$. Let us denote the embedding by
$\iota$. Now, define the distance on $\mathcal{U}^M$ as follows:

\begin{align}
d\left(\left\{\widetilde{X}_i\right\}_{i=1}^M,
\left\{Y_j\right\}_{j=1}^M\right) = \max_{i,j} d\left(X_i, Y_j\right)
\end{align}

The choice of weights for $\textsf{wFM}$ is said to be trivial if one
of the weights is $1$ and hence the rest are $0$.

\begin{proposition}
\label{theory:prop2}
For all nontrivial choices of $\left\{\alpha_i\right\}_{i=1}^N$ and
$\left\{\beta_j\right\}_{j=1}^M$ satisfying the convexity constraint ,
$\exists c < 1$ such that,
\begin{align}
d\left(\textsf{wFM}\left(\left\{X_i\right\}_{i=1}^N, \left\{\alpha_i\right\}_{i=1}^N\right), \textsf{wFM}\left(\left\{Y_j\right\}_{i=1}^M, \left\{\beta_j\right\}_{i=1}^M\right)\right) \leq c\;d\left(\iota\left(\left\{X_i\right\}_{i=1}^N\right), \left\{Y_j\right\}_{j=1}^M\right)
\end{align}
\end{proposition}
\begin{proof}
This is easy to prove and hence we skip the proof here.
\end{proof}

\subsection{The invariant (last) layer}
\label{sec23}
We will form a deep network by cascading multiple $\textsf{wFM}$
blocks each of which acts as a convolution-type layer. Each convolutional-type
layer is equivariant to the group action, and hence at the end of
the cascaded convolutional layers, the output is equivariant to the
group action applied to the input of the network. Let $d$ be the
number of output channels each of which outputs a $\textsf{wFM}$,
hence each of the channels is equivariant to the group
action. However, in order to build a network that yields an output
which is invariant to the group action, we now seek the last layer
(i.e., the linear classifier) to be invariant to the group action. The
last layer is thus constructed as follows: Let $\left\{Z_1, \cdots,
Z_d\right\} \subset \mathcal{M}$ be the output of $d$ channels and
$M_u = \textsf{FM}\left(\left\{Z_i\right\}_{i=1}^d\right)$
$=\textsf{wFM}\left(\left\{Z_i\right\}_{i=1}^d,
\left\{\sfrac{1}{d}\right\}_{1}^d\right)$ be the unweighted FM of the
outputs $\left\{Z_i\right\}_{i=1}^d$. Then, we construct a layer with
$d$ outputs whose $i^{th}$ output $o_i = d\left(M_u, Z_i\right)$. Let
$c$ be the number of classes for the classification task, then, a
fully connected (FC) layer with inputs $\left\{o_i\right\}$ and $c$
output nodes is build. A softmax operation is then used at the $c$
output nodes to obtain the outputs $\left\{y_i\right\}_{i=1}^c$. In
the following proposition we claim that this last layer with
$\left\{Z_i\right\}_{i=1}^d$ inputs and $\left\{y_i\right\}_{i=1}^c$
outputs is group invariant.

\begin{proposition}
The last layer with $\left\{Z_i\right\}_{i=1}^d$ inputs and
$\left\{y_i\right\}_{i=1}^c$ outputs is group invariant.
\end{proposition}
\begin{proof}
Using the above construction, let $W \in \mathbf{R}^{c \times d}$ and
$\mathbf{b} \in \mathbf{R}^c$ be the weight matrix and bias
respectively of the FC layer. Then,
\begin{align*}
\mathbf{y} &= F\left(W^T\mathbf{o} + \mathbf{b}\right) \\
&= F\left(W^T d\left(M_u, Z\right) + \mathbf{b}\right), \numberthis
\end{align*}
where, $F$ is the softmax function. In the above equation, we treat
$d\left(M_u, Z\right)$ as the vector $\left[d\left(M_u, Z_1\right),
  \cdots, d\left(M_u, Z_d\right)\right]^t$. Observe that, $g.M_u =
\textsf{FM}\left(\left\{g.Z_i\right\}_{i=1}^d\right)$. As each of the $d$
channels is group equivariant, $Z_i$ becomes $g.Z_i$. Because of the
property of the distance under group action, $d\left(g.M_u,
g.Z_i\right) = d\left(M_u, Z_i\right)$. Hence, one can see that if we
change the inputs $\left\{Z_i\right\}$ to $\left\{g.Z_i\right\}$, the
output $\mathbf{y}$ will remain invariant.
\end{proof}
In Fig. \ref{fig0} we present a schematic of ManifoldNet depicting the
different layers of processing the manifold-valued data as described
above in Sections \ref{sec21}-\ref{sec23}.
 \begin{figure}[!ht]
\centering
\includegraphics[scale=0.6]{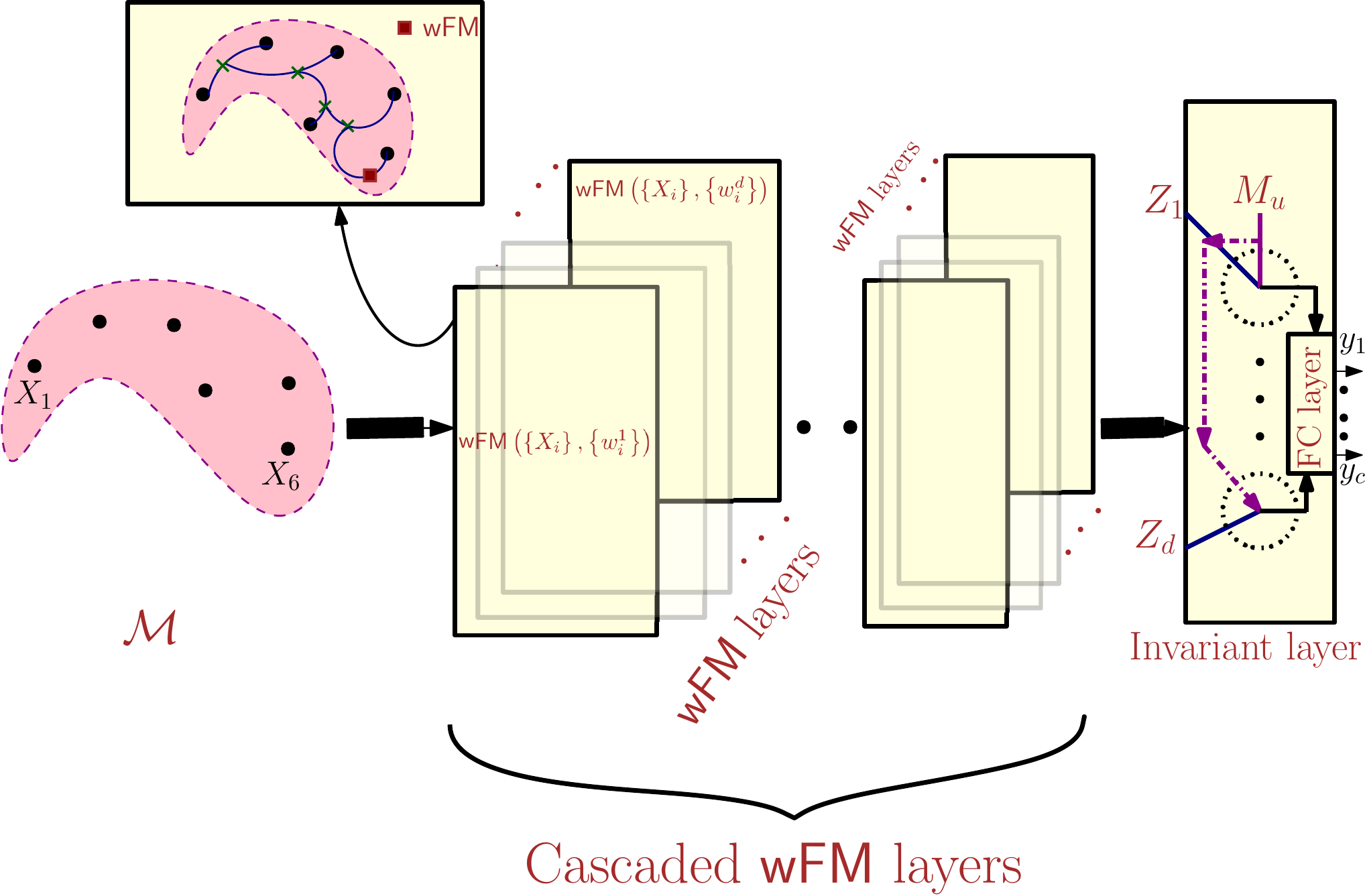}
\caption{Schematic diagram of ManifoldNet}
\label{fig0}
\end{figure}

\section{Experiments}\label{results}

In this section we present performance of the ManifoldNet framework on
several computer vision problems.  The breadth of application coverage
here includes classification and reconstruction problems. We begin
with a video classification problem and then present a reconstruction
problem using an auto-encoder-decoder set up.

\subsection{Video Classification}
We start by using the method in \cite{Yu2017SecondOrderCNN} which we
summarize here.  Given a video with dimensions $F \times 3 \times H
\times W$ of $F$ frames, $3$ color channels and a frame size of $H
\times W$, we can apply a convolution layer to obtain an output of
size $F \times C \times H' \times W'$ consisting of $C$ channels of
size $H' \times W'$. Interpreting each channel as a feature map, we
shift the features to have a zero mean and compute the covariance
matrix of the convolution output to obtain a sequence of $F$ symmetric
positive (semi) definite (SPD) matrices of size $C \times C$. From
here we can apply a series of temporal ManifoldNet $\textsf{wFM}$s to
transform the $F \times C \times C$ input to a temporally shorter $F'
\times K \times C \times C$ output, where $K$ are the temporal
$\textsf{wFM}$ channels. Within the temporal ManifoldNet
$\textsf{wFM}$s we use a simple weight normalization to ensure that
the weights are within $[0,1]$, and for the weights $w_i$ of any
output channel we add a weight penalty of the form $(\sum w_i-1)^2$ to
the loss function to ensure that we obtain a proper $\textsf{wFM}$. We
then reshape this to $F'K \times C \times C$ and pass it through an
invariant final layer (section \ref{sec23}) to obtain a vector of size
$F'K$. Finally, a single FC+SoftMax layer is applied to produce a
classification output. We call this the SPD temporal convolutional
architecture SPD-TCN. Figure \ref{ManifoldTCN} illustrates the network
architecture described above. In general, the SPD-TCN tends to perform
very well on video classification tasks while \emph{using very few
  parameters}, and runs efficiently due to the $\textsf{wFM}$
structure.

\begin{figure}[!ht]
\label{ManifoldTCN}
\centering
\hbox{\hspace{-4em} \includegraphics[clip,trim=0cm 3cm 0cm 8cm,scale=0.6]{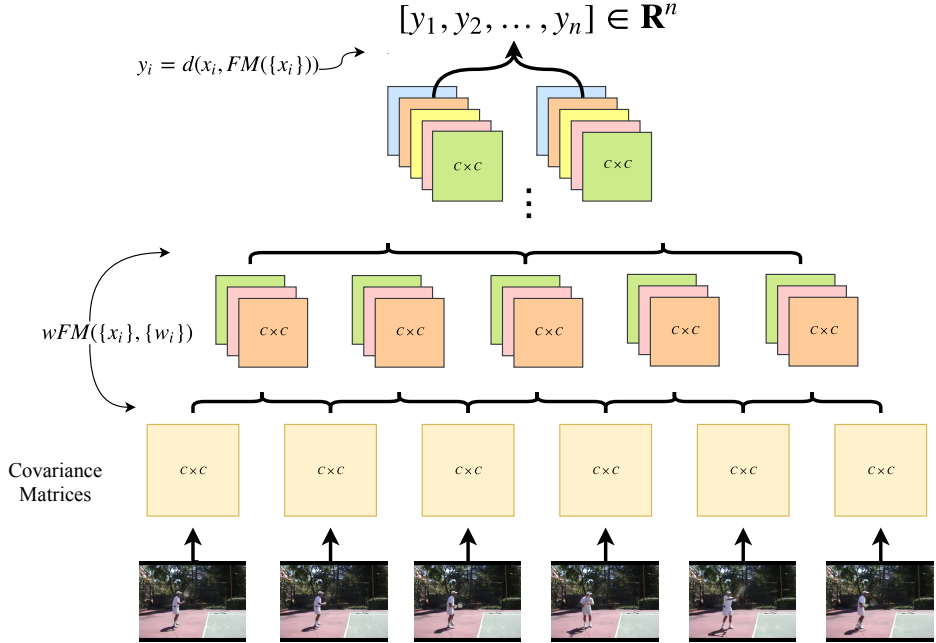}}
\caption{ManifoldTCN Network Architecture}
\label{fig1}
\end{figure}

We tested the ManifoldTCN on the Moving MNIST dataset
\cite{Srivastava2015LSTM-Vid}. In \cite{Chakraborty2018ManifoldRNN}
authors developed a manifold valued recurrent network architecture,
dubbed SPD-SRU, which produced state of the art classification
results on a version of the Moving MNIST dataset in comparison to LSTM
\cite{Hochreiter1997LSTM}, SRU \cite{Oliva2017SRU}, TT-LSTM and TT-GRU
\cite{Yang2017TensorTrain} networks.  For the LSTM and SRU networks,
convolution layers are also used before the recurrent unit. We will
compare directly with these results. For details of the various
architectures used please see section 5 of
\cite{Chakraborty2018ManifoldRNN}.  The Moving MNIST data generated in
\cite{Srivastava2015LSTM-Vid} consists of $1000$ samples, each of $20$
frames. Each sample shows two randomly chosen MNIST digits moving
within a $64 \times 64$ frame, with the direction and speed of
movement fixed across all samples in a class. The speed is kept the
same across different classes, but the digit orientation will differ
across two different classes. We summarize the $10$-fold cross
validation results for several orientation differences between classes
in Table \ref{results:tab1}.  For this experiment the SPD-TCN will
consist of a single $\textsf{wFM}$ layer with kernel size $5$ and stride
$3$ returning $8$ channels, making for an $8 \times 8$ covariance
matrix. We then apply three temporal SPD $\textsf{wFM}$ layers of kernel
size $3$ and stride $2$, with the following channels $1 \to 4 \to 8
\to 16$, i.e. after these three temporal SPD $\textsf{wFM}$s we have $16$
temporal channels. This $16 \times 8 \times 8$ is used as an input to
the invariant final layer to get a $16$ dimensional output vector,
which is transformed by a fully connected layer and SoftMax to obtain
the output.

\begin{table}[!ht]
	\label{video-res}
   \centering
   \scalebox{0.9}{
\begin{tabular}{cccccc} 
\topline\rowcolor{tableheadcolor}
 & & {\bf time (s)} & \multicolumn{3}{c}{{\bf orientation ($^\circ$)}} \\
\arrayrulecolor{tableheadcolor}\hhline{---~~~}\arrayrulecolor{rulecolor}\hhline{~~~---}\rowcolor{tableheadcolor}
\multirow{-2}{*}{\bf Mode} & \multirow{-2}{*}{\bf \# params.} & {\bf / epoch} & $30$-$60$ & $10$-$15$ & $10$-$15$-$20$ \\
\midtopline
SPD-TCN & $\highest{738}$ & $\sim 2.7$ & $\highest{1.00 \pm 0.00}$ & $\highest{0.99 \pm 0.01}$  & $\highest{0.97 \pm 0.02}$\\
SPD-SRU & $1559$ & $\sim 6.2$ & $\highest{1.00\pm 0.00}$ & ${0.96 \pm 0.02}$ & ${0.94 \pm 0.02}$ \\
TT-GRU & $2240$ & $\sim\highest{2.0}$ & $\highest{1.00 \pm 0.00}$ & $0.52 \pm 0.04$ & $0.47 \pm 0.03$ \\
TT-LSTM & $2304$ & $\sim\highest{2.0}$ & $\highest{1.00 \pm 0.00}$ & $0.51 \pm 0.04$ & $0.37 \pm 0.02$ \\
SRU & $159862$ & $\sim 3.5$ & $\highest{1.00 \pm 0.00}$ & $0.75 \pm 0.19$ & $0.73 \pm 0.14$ \\
LSTM & $252342$ & $\sim 4.5$ & $0.97 \pm 0.01$ & $0.71 \pm 0.07$ & $0.57 \pm 0.13$ \\
\bottomline
\end{tabular}
}
\caption{\footnotesize Comparison results on Moving MNIST}
\label{results:tab1}
\end{table}

\subsection{Dimensionality Reduction}

Here we present experiments demonstrating the applicability of the
theory layed out in Section \ref{theory} to the case of linear
dimensionality reduction, specifically principal component analysis
(PCA), which is the workhorse of many machine learning algorithms. In
\cite{Chakraborty2017Grassmann}, authors presented an online subspace
averaging algorithm for construction of principal components via
intrinsic averaging on the Grassmannian. In this section, we achieve
the intrinsic Grassmann averaging process in the framework of
ManifoldNets to compute the principal subspaces and achieve
the dimensionality reduction. In the context of deep neural networks,
dimensionality reduction is commonly achieved via an autoencoder
architecture.
More recently, deep neural networks have shown promising results when
the data manifold is intrinsically non-linear, as in the case of
natural images. In the deep learning community this has become a field
in its own right, known as representation learning or feature learning
\cite{Bengio2012Representations}. This field has seen several
significant advances in the past few years, including the introduction
of denoising autoencoders \cite{Vincent2010Denoising}, variational
autoencoders \cite{Kingma2013VAE}, autoregressive models (PixelCNN
\cite{Oord2016pixelcnn}, PixelRNN \cite{Oord2016pixelrnn}), and
flow-based generative models \cite{Kingma2018GLOW}. Many of these
architectures are modifications of the traditional autoencoder
network, which involves learning an identity map through a small
latent space. In our application, we modify the traditional
autoencoder model by adding a ManifoldNet layer to perform a learned
linear dimensionality reduction in the latent space, although in
principal, our techniques can be applied to most autoencoder based
models such as the variational autoencoders.

To compute a linear subspace in the ManifoldNet framework we use an
intrinsic averaging scheme on the Grassmannian. Points on the
Grassmannian $\textsf{Gr}(k,n)$ correspond to $k$-dimensional
subspaces of the vector space $\mathbf{R}^n$. The Grassmannian is a
smooth homogeneous manifold and points $\mathcal{X} \in
\textsf{Gr}(k,n)$ on the Grassmannian can be specified by an
orthonormal basis $X$, i.e. an $n \times k$ orthonormal
matrix. Hauberg et al. showed that the one dimensional principal
subspace can be computed as an average of all one dimensional
subspaces spanned by normally distributed data
\cite{Hauberg2016PCA}. Motivated by this result, Chakraborty et
al. proposed an efficient intrinsic averaging scheme on
$\textsf{Gr}(k,n)$ that converges to the $k$-dimensional principal
subspace of a normally distributed dataset in $\mathbf{R}^n$
\cite{Chakraborty2017Grassmann}. In the ManifoldNet framework we can
modify this technique to learn a $\textsf{wFM}$ of points on the
Grassmannian that corresponds to a subspace of the latent space which
minimizes the reconstruction error by using a Grassmannian averaging
layer that learns the weights in the $\textsf{wFM}$. This essentially
will give us a lower dimensional representation of the samples after
projecting them on to the learned subspace. Note that combining the
convergence proof in Chakraborty et
al. \cite{Chakraborty2017Grassmann} and Proposition
\ref{theory:prop1}, we claim that the $\textsf{wFM}$ learned using the
ManifoldNet asymptotically converges to the principal subspace. In the
rest of the section, we give a detailed description of our
experimental setup to show the applicability of ManifoldNet to
dimensionality reduction.
\subsubsection{Implementation Details}

A traditional convolutional autoencoder performs non-linear
dimensionality reduction by learning an identity function through a
small latent space. A common technique used when the desired latent
space is smaller than the output of the encoder is to apply a fully
connected layer to match dimensions. We replace this fully connected
layer by a weighted subspace averaging and projection block, called
the Grassmann averaging layer. Specifically, we compute the 
$\textsf{wFM}$ of the output of the encoder to get a subspace in the
encoder output space.  We then project the encoder output onto this
space to obtain a reduced dimensionality latent space.  We explicitly
perform linear dimensionality reduction in the latent space, hence, we
should expect the encoder to linearize the data manifold to some
degree, followed by a linear dimensionality reduction from the
Grassmann averaging layer. When this layer is used in a larger network
we add a weight penalty to the loss function of the form $(\sum w_i -
1)^2$ to ensure that we compute a proper $\textsf{wFM}$.  In
general this offers a significant parameter reduction while also
increasing the reconstruction error performance of the autoencoder and
giving realistic reconstructions.  We call an autoencoder with the
Grassmann averaging block an autoencoder+iFME network, as shown in
Fig. \ref{fig0.5}. In the experiments we compare this to other
dimensionality reduction techniques, including regular autoencoders
that use fully connected layers to match encoder and latent space
dimensions. The PyTorch code used to run some of these experiments,
including a PyTorch module for easily using a Grassmann averaging
layer are accessible via GitHub
\footnote{https://github.com/jjbouza/manifold-net}.

 \begin{figure}[!ht]
\centering
\includegraphics[scale=0.58]{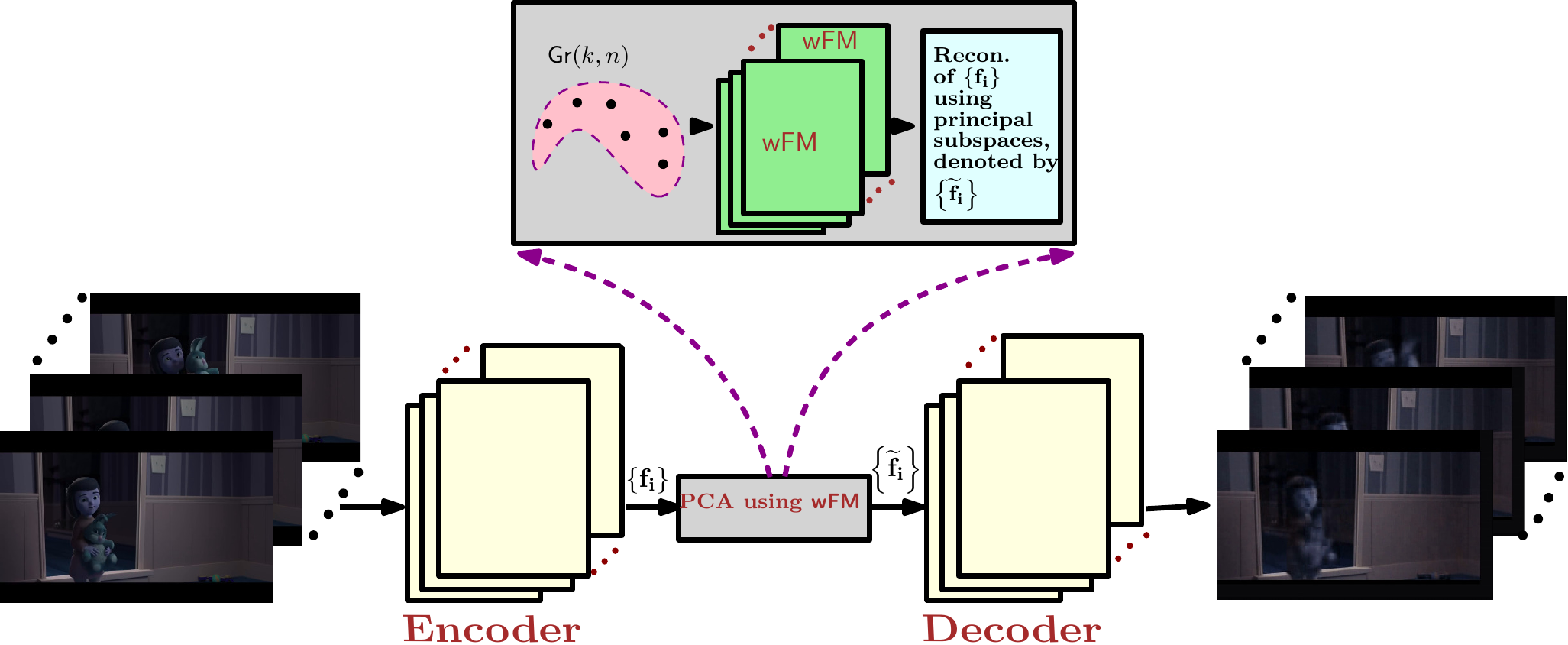}
\caption{Schematic description of autoencoder+iFME network}
\label{fig0.5}
\end{figure}
\subsubsection{Video Reconstruction Experiment}

Assume we start with an $N \times 3 \times W \times H$ video, where
$N$,$W$, and $H$ are the number of frames, the frame width and the
frame height respectively. Then, we apply an encoder in parallel to
the frames to obtain a $N \times K \times M$ feature video $F$, where
$K$ and $M$ are the channels and feature dimension,
respectively. Finally, we compute the $\textsf{wFM}$ of $F$ as $N$
vectors in $\mathbf{R}^{KM}$ and project $F$ onto this subspace to get
a reduced dimensionality latent space.  From here we proceed as in a
regular autoencoder, backpropagating the $L^2$ reconstruction error
throughout the network.  Note that here we can interpret the Grassmann
averaging layer weights as indicating the relative importance of
different temporal segments of the video for the purpose of learning
reconstructable features.  To compare with a regular autoencoder we
simply remove the Grassmann averaging layer and replace it with a
fully connected layer that maps to the same latent space dimension.
This is justified by the fact that both operations are linear on the
data, i.e. the Grassman averaging block is not adding additional
non-linear capacity to the network.

We begin by testing on a $1000$ frame color sample of video from the
$1964$ film “Santa Clause Conquers the Martians” of frame size $320
\times 240$. Here we use an $8$ layer encoding-decoding architecture
with $\textbf{Conv} \rightarrow \textbf{ELU} \rightarrow
\textbf{Batchnorm}$ layers, with the final layer applying a sigmoid
activation to normalize pixel values.

The encoder returns a feature video consisting of $128$ channels of
size $120$ for a dimension of $1000 \times 15360$.  We compare a fully
connected layer to a Grassmann averaging layer, both mapping to a
desired latent space of dimension $1000 \times 20$.  The per pixel
average reconstruction error for the Grassmann block network is
$0.0110$, compared to $0.0122$ for the fully connected network,
representing an improvement of $10.9 \%$. In general, the Grassmann
averaging layer tends to do as well or better than the fully connected
layer. Although in theory the fully connected layer can learn the same
mapping as the Grassmann averaging layer, it has a much larger
parameter space to search for this solution, implying that it is more
likely to get trapped in local minima in the low loss regions of the
loss surface.  We also observe a parameter reduction of $46 \%$, which
can be attributed to the number of parameters in the large fully
connected layer.  In general the Grassmann averaging layer network is
slower per iteration than the fully connected network, but also tends
to exhibit faster convergence so that the time to reach the same
reconstruction error is less for the Grassmann averaging layer. This
is depicted in Fig. \ref{fig1}, where, computation time is plotted
against error tolerance for the autoencoder and the iFME+autoencoder.
Overall, we see an improvement in all major performance categories.

 \begin{figure}[!ht]
\centering
\includegraphics[scale=0.38]{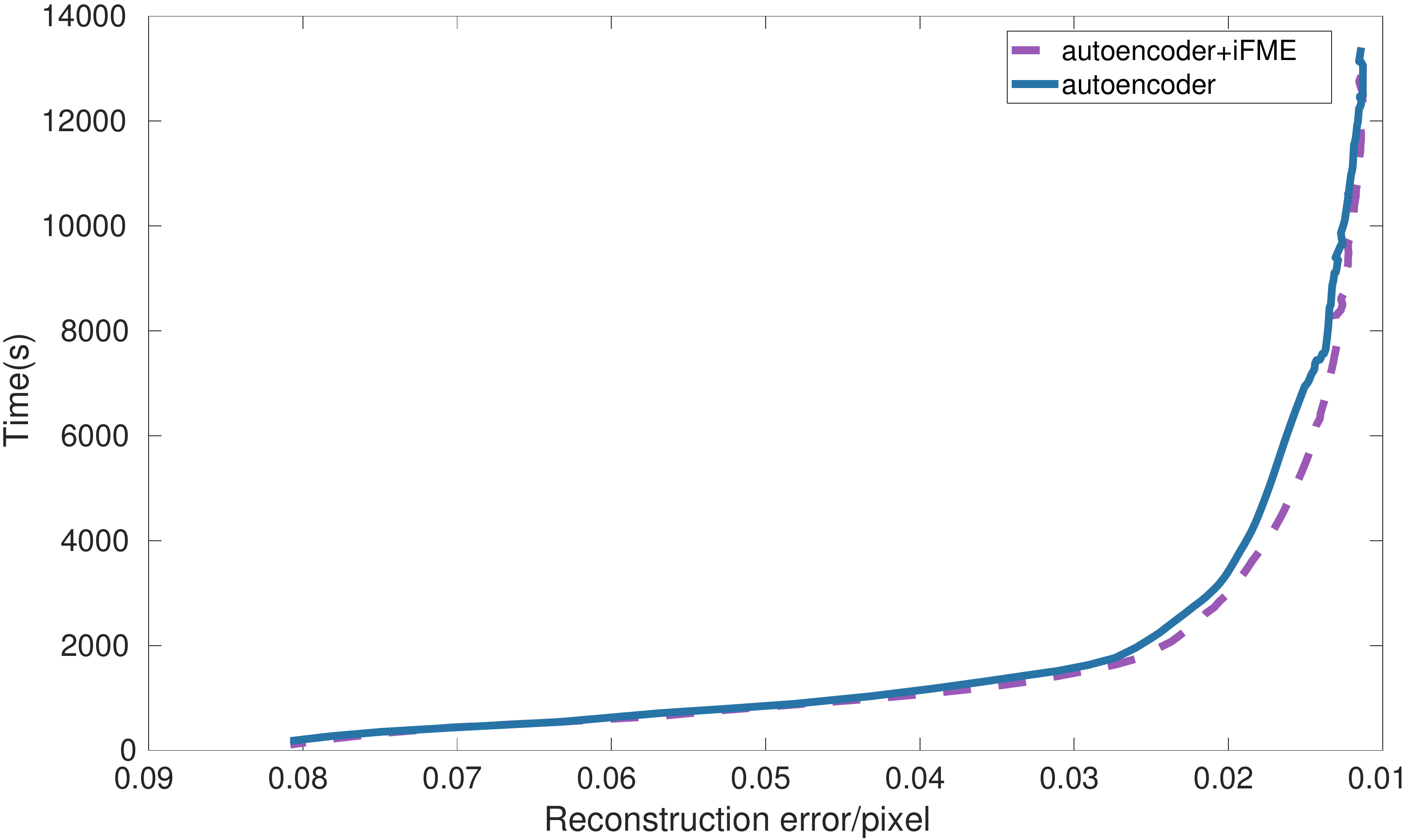}
\caption{Computation time vs. error tolerance plot comparison between autoencoder and iFME+autoencoder}
\label{fig1}
\end{figure}

It is possible to obtain a low reconstruction error on autoencoding
tasks and still observe low visual quality reconstructions.  To ensure
this is not the case we run the same experiment on 300 frames of the
$1280 \times 720$ short
film \footnote{https://www.youtube.com/watch?v=t1hMBnIMt5I}, with a
latent space frame dimension of 300x50. In Fig. \ref{fig2} we compare
the visual quality of our autoencoder to that of PCA with $50$
principcal components, i.e., we reduce the dimension from $1280\times
720 \times 3$ to $50$.  The entire sample reconstruction is
shown in \footnote{https://streamable.com/3yqrx} (in the same order as
in Fig. \ref{fig2}).

\begin{figure}[!ht]
\centering
\includegraphics[scale=0.40]{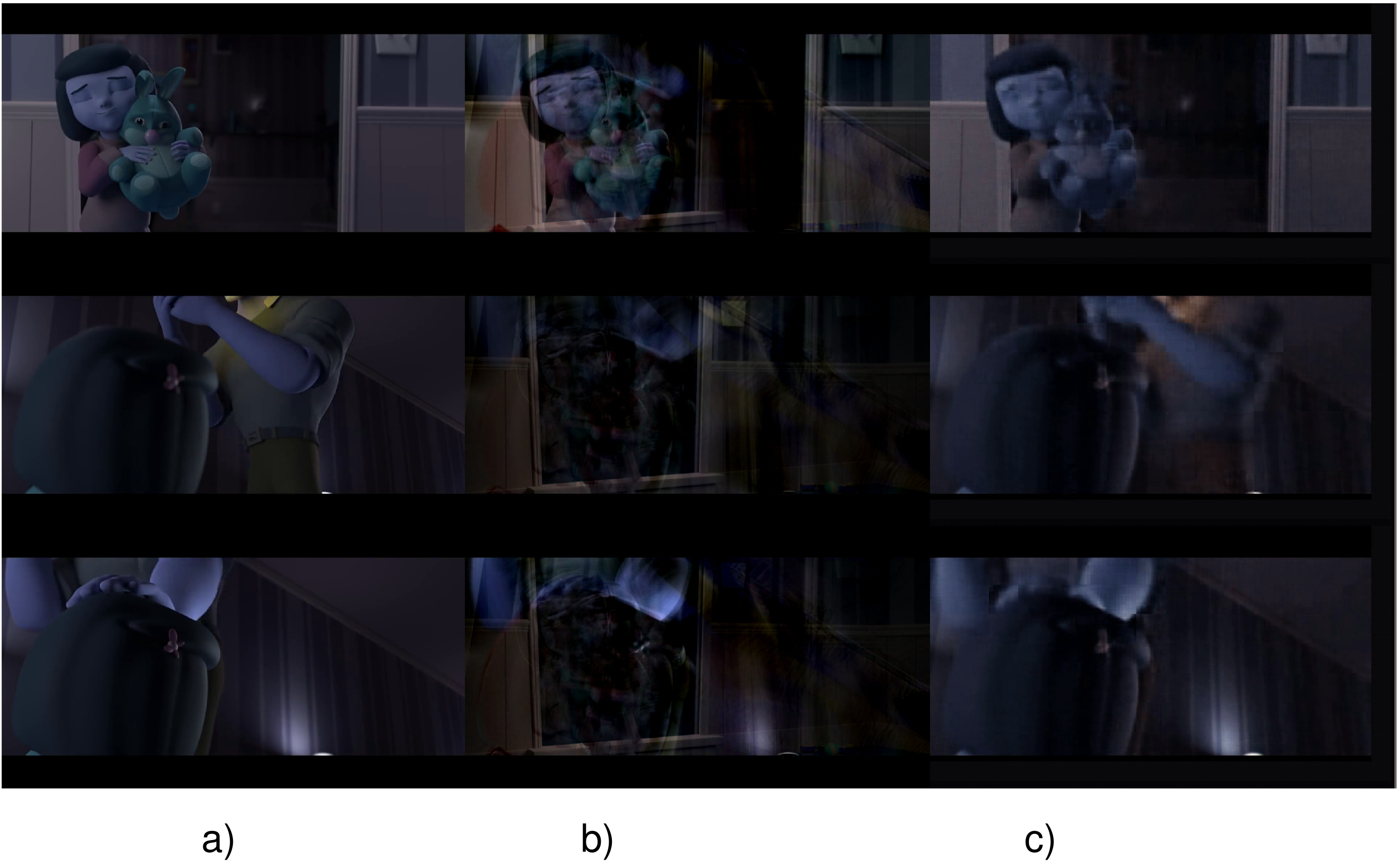}
\caption{Reconstruction of select movie frames (a) original frame (b) using PCA (c) using iFME+autoencoder}
\label{fig2}
\end{figure}

\subsubsection{MNIST Reconstruction Experiment}

Finally, we applied the ManifolNet framework to MNIST
data \footnote{http://yann.lecun.com/exdb/mnist/}. We trained a 4-layer
autoencoder on a random sample of $1000$ images and test the
reconstruction error on a $100$ sample test set. Here the Grassmann
averaging layer reduces from an encoder dimension of $25$ to a latent
space of dimension $2$. In Fig. \ref{fig:mnist_reconstruction} we
compared the reconstructions of our autoencoder+iFME to a regular
autoencoder with a fully connected layer after the encoder and to a
PCA reconstruction in image space.

\begin{figure}[!ht]
	\centering
	\includegraphics[scale=0.5]{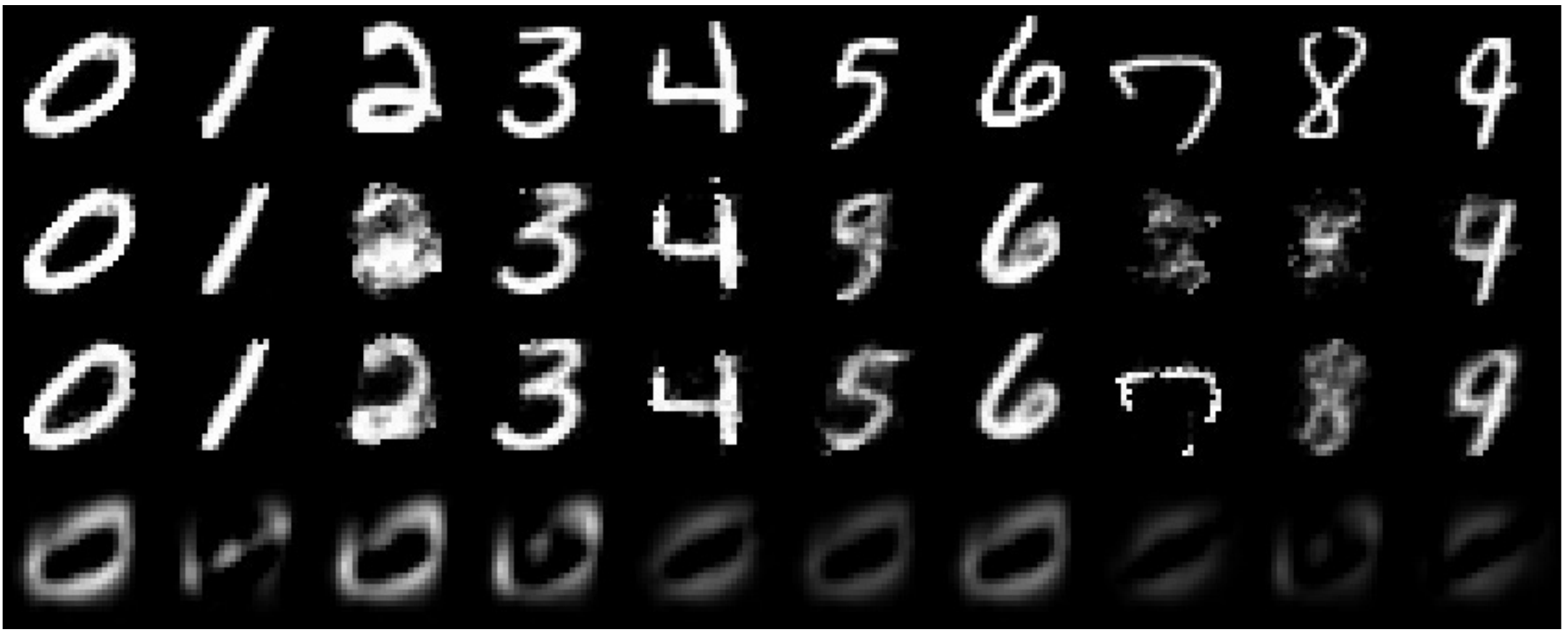}
	\caption{Reconstructions of (a), a sample of the validation set: (b) Regular autoencoder (c) Autoencoder+iFME (d) PCA }
	\label{fig:mnist_reconstruction}
\end{figure}

To evaluate the validation set reconstruction performance rigorously
we trained a simple two layer fully connected classifier on the entire
MNIST data set, achieving a training accuracy of $99.97\%$ after $120$
epochs of training. We then applied this classifier to all of the test
set reconstructions given by the above methods, and evaluated the
accuracy of the reconstructions by measuring the classification
accuracy of the classifier. These results are summarized in
Fig. \ref{fig:mnist_reconstruction}.

We observe that a regular autoencoder yielded marginally better
classification accuracy on the training set, but autoencoder+iFME
gave a much larger improvement on the test set. Since the fully
connected layer at the end of the regular autoencoder can learn a
mapping identical to the projection in our Grassmann averaging layer,
we conclude that our method acts as a regularizer here, preventing
overfitting to the training set.

\begin{figure}[!ht]
	\centering

	\begin{tabular}{l|l|c|r}
		 & Autoencoder & Autoencoder+iFME & PCA \\
		\hline
		 Training Set & $\mathbf{0.862}$ & $0.8540$ & $0.2300$ \\
		 Validation Set & $0.4000$ & $\mathbf{0.4600}$ & $0.2600$ \\
	\end{tabular}
	\caption{Classification Accuracy of MNIST data Classifier on
          Reconstructed Data}

	\label{fig:mnist_reconstruction}
\end{figure}

\section{Conclusions}\label{conc}
In this paper, we presented a novel deep network called ManifoldNet
suited for processing manifold-valued data sets. Manifold-valued data
are commonly encountered in many computer vision and medical imaging
applications. Examples of such data include but are not limited to
directional data which reside on a sphere, covariance descriptors
which reside on the manifold of symmetric positive matrices and
others.  Note that inputs to ManifoldNet are manifold-valued and not
real or complex-valued functions defined on non-Euclidean domains.
Our key contributions are: (i) A novel deep network that maybe
perceived as being a generalization of the CNN to the case when the
input data are manifold-valued using purely intrinsic operations that
are natural to the manifold on which the data reside.  (ii) Analogous
to convolutions in vector spaces -- which can be achieved via
computation of weighted mean -- we present weighted FM operations as a
replacement for the convolutions in CNNs and prove a theorem on the
equivariance of the weighted FM to natural group operations admitted
by the manifold. This equivariance allows us to share the learned
weights within a layer of the ManifoldNet. (iii) An efficient
recursive weighted FM estimator that is provably convergent is
presented. (iv) Experimental results demonstrating the efficacy of
ManifoldNets for, (a) video classification and (b) principal component
computation from videos and reconstruction are also presented.

{\small
\bibliographystyle{plain}
\bibliography{references}
}

\end{document}